\documentclass{llncs}
\usepackage{fullpage}
\usepackage{natbib}
\usepackage[linesnumbered]{algorithm2e}

\usepackage{booktabs}  
\usepackage{caption}
\usepackage{booktabs}  
\usepackage{mathrsfs}
\usepackage{tikz}
\usepackage{bbding,pifont}
\usepackage{pgflibraryshapes}
\usepackage{graphicx}
\usepackage{hyperref}
\usepackage{paralist}
\usetikzlibrary{trees}
\usetikzlibrary{positioning,chains,fit,shapes,calc}
\usepackage{eucal}
\usepackage{nicefrac}
\usepackage{balance}
\usepackage{tikz}
\tikzset{
  treenode/.style = {align=center, inner sep=0pt, text centered,
    font=\sffamily},
  arn_n/.style = {treenode, circle, white, font=\sffamily\bfseries, draw=black,
    fill=black, text width=1.5em},
  arn_r/.style = {treenode, circle, red, draw=red,
    text width=1.5em, very thick},
  arn_x/.style = {treenode, rectangle, draw=black,
    minimum width=0.5em, minimum height=0.5em}
}

\usepackage{array,xspace,multirow,hhline,tikz,colortbl,tabularx,booktabs,fixltx2e,amsmath,amssymb,amsfonts,amsthm}
\usepackage{natbib}
\usepackage{verbatim,ifthen}
\usepackage{pifont}
\usepackage{ifthen}
\usepackage{times}
\usepackage{calrsfs,mathrsfs}
\usepackage{bbding,pifont}
\usepackage{pgflibraryshapes}
\usetikzlibrary{trees}
\usetikzlibrary{positioning,chains,fit,shapes,calc}

\usepackage{eucal}

\usepackage{tikz}
\usetikzlibrary{arrows}
\usetikzlibrary{decorations.pathreplacing}

\usepackage{amsmath}

	\usepackage{varioref}

\usepackage{algorithmic}

\newcommand\eat[1]{}

\newlength{\wordlength}

\newcommand{\ms}{\mathcal S}

\newcommand{\mo}{\pi}

\newcommand{\ra}{\rightarrow}

\newcommand{\Omit}[1]{}

\newtheorem{thm}{Theorem}

\newtheorem{coro}{Corollary}

\newtheorem{cond}{Condition}

		\newcommand{\pref}{\ensuremath{\succsim}\xspace}

			\newcommand{\possibleitem}{{\sc PossibleItem}\xspace}
				\newcommand{\necessaryitem}{{\sc NecessaryItem}\xspace}
					\newcommand{\possibleset}{{\sc PossibleSet}\xspace}
						\newcommand{\necessaryset}{{\sc NecessarySet}\xspace}
					\newcommand{\possiblesubsetallocation}{{\sc PossibleSubset}\xspace}
						\newcommand{\necessarysubset}{{\sc NecessarySubset}\xspace}
				\newcommand{\possibleassignment}{{\sc PossibleAssignment}\xspace}
						\newcommand{\necessaryassignment}{{\sc NecessaryAssignment}\xspace}

\begin{document}

 \title{Possible and Necessary Allocations \\via Sequential Mechanisms}

\author{%
	Haris Aziz \and Toby Walsh \and Lirong Xia}

	\institute{%
	NICTA and UNSW, Sydney, NSW 2033, Australia\\
		\email{haris.aziz@nicta.com.au}\\
		\and
	NICTA and UNSW, Sydney, NSW 2033, Australia\\
		\email{toby.walsh@nicta.com.au}\\
		\and
		RPI, NY 12180, USA\\
		\email{xial@cs.rpi.edu}\\
	}

 %
 

\maketitle
\begin{abstract}
A simple mechanism for allocating indivisible resources
is sequential allocation in which agents take turns to pick items. We focus on {\em possible} and {\em necessary allocation} problems,
checking whether allocations of a given form
occur in {\em some} or {\em all} mechanisms
for several commonly used classes of sequential 
allocation mechanisms.
In particular, we consider whether a given agent receives a given item, a set of items, or a subset of items for 
five natural classes of sequential allocation mechanisms: 
balanced, recursively balanced, balanced alternating, strictly alternating and all
policies.
We identify characterizations of allocations produced balanced, recursively balanced, balanced alternating policies and strictly alternating policies respectively, which extend the well-known characterization by~\citet{BrKi05a} for policies without restrictions. In addition, we examine the
computational complexity of possible and necessary allocation problems
for these classes. 
 \end{abstract}

	\section{Introduction}

	Efficient and fair allocation of resources
	is a pressing problem within society today.
	One important and challenging case is the
	fair allocation of indivisible items~\citep{CDE+06a,BoLa08a,BEL10a,AGMW14a,Aziz14a}. This
	covers a wide range of problems
	including the allocation of classes to students,
	landing slots to airlines, players to teams, and houses to
	people. 
	A simple but popular mechanism to allocate
	indivisible items is {\em sequential allocation} 
	\citep{BoLa11a,BrTa96a,KoCh71a,LeSt12a}. 
	In sequential allocation, agents simply take turns to pick 
	the most preferred item that has not yet been taken. 
	Besides its simplicity, it has a number of
	advantages including the fact that the mechanism can be implemented in a distributed manner and that agents do not need to submit cardinal utilities. Well-known  mechanisms like serial dictatorship~\citep{Sven99a} fall under the umbrella of sequential mechanisms.

	The sequential allocation mechanism leaves open the particular
	order of turns (the so called ``policy'')~\citep{KNW13a,BoLa14b}.
	Should it be a \emph{balanced} policy i.e., each agent gets the same 
	total number of turns?
	Or should it be \emph{recursively balanced} so that turns
	occur in rounds, and each agent gets one turn per round?
	Or perhaps it would be fairer to alternate but reverse the order
	of the agents in successive rounds: $a_1\rhd a_2\rhd a_3\rhd a_3\rhd a_2\rhd a_1\ldots$  so that agent $a_1$ takes the first and sixth turn? This particular type of policy is used, for example, by the Harvard Business School to allocate courses to students~\citep{BuCa12a} and is referred to as a \emph{balanced alternation} policy.
Another class of policies is \emph{strict alternation} in which the same ordering is used in each round, such as $a_1\rhd a_2\rhd a_3\rhd a_1\rhd a_2\rhd a_3\ldots$ . 
	The sets of balanced alternation  and strict alternation 
	policies are subsets of the set of recursively balanced policies which 
	itself is a subset of the set of balanced policies (see Figure~\ref{fig:inclfigure}).

	We consider here the situation where a policy is chosen from 
	a family of such policies. For example, at the Harvard 
	Business School, a policy is chosen at random from the space
	of all balanced alternation policies. 
	As a second example, the policy might be left to the
	discretion of the chair but, for fairness, it
	is restricted to one of the recursively balanced policies. 
	Despite uncertainty in the policy, we might be interested in the
	possible or necessary outcomes. For example, can I get
	my three most preferred courses? Do I necessarily
	get my two most preferred courses? 
	We examine the complexity of checking 
	such questions. There are several high-stake 
	applications for these results. For example, sequential allocation is used in professional sports `drafts'~\citep{BrSt79a}. The precise policy chosen from among the set of admissible policies can critically affect which teams (read agents) get which players (read items).

	The problems of checking whether an agent can get some item or set of items in a policy or in all policies is closely related to the problem of `control' of the central organizer. For example, if an agent gets an item in all feasible policies, then it means that the chair cannot ensure that the agent does not get the item.
	Apart from strategic motivation, the problems we consider also have a design motivation. The central designer may want to consider all feasible policies uniformly at random (as is the case in random serial dictatorship~\citep{ABB13b,SaSe13a}) and use them to find the probability that a certain item or set of item is given to an agent. The probability can be a suggestion of time sharing of an item. The problem of checking whether an agent gets a certain item or set of items in some policy is equivalent to checking  whether an agent gets a certain item or set of items with non-zero probability. Similarly, the problem of checking whether an agent gets a certain item or set of items in all policy is equivalent to checking  whether an agent gets a certain item or set of items with probability one.

						\begin{figure*}
							\centering
					     \begin{tikzpicture}[scale=0.3]
					    \tikzstyle{every circle node}=[draw,inner xsep=1.5em]
					    \draw[opacity=0.5]
					     (-13,-8) rectangle (13,8);

					    \draw[draw=black,rotate=360,opacity=0.5] (0,0) ellipse (35em and 22em);

										    \draw[draw=black,rotate=360,opacity=0.5] (0,1) ellipse (26em and 13em);

					 \draw[draw=black,rotate=360,opacity=0.5] (3,1) ellipse (15em and 8em);

					 				 \draw[draw=black,rotate=360,opacity=0.5] (-3,1) ellipse (15em and 8em);
								 

				  \draw(-10,-7) node {\small Arbitrary};
				  \draw(7,6) node {\small Balanced};

				   \draw(0,4.5) node {\small Rec-Balanced};

				    \draw(-5,1) node {\small Strict-Alt};
					\draw(5,1) node {\small Bal-Alt};
				



					     \end{tikzpicture}
	\caption{Inclusion relationships between sets of policies. We use abbreviations Rec-Balanced (recursively balanced); Strict-Alt (strict alternation), and Bal-Alt (balanced alternation). }\label{fig:inclfigure}
					\end{figure*}
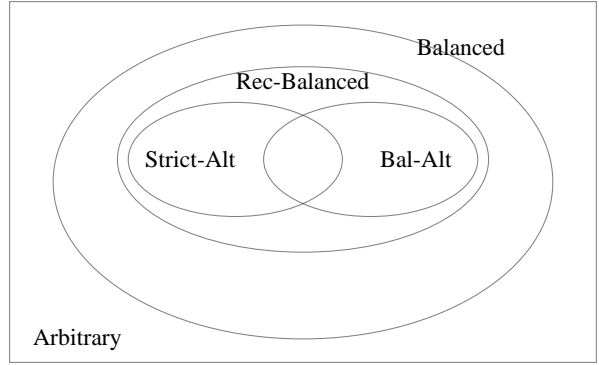

	We let $A=\{a_1,\ldots, a_n\}$ denote a set of $n$ agents, and $I$ denote the set of $m=kn$ items\footnote{This is without loss of generality since we can add dummy items of no utility to any agent.}. 
	$P=(P_1,\ldots,P_n)$ is the profile of agents' preferences where each $P_j$ is a linear order over $I$. Let $M$ denote an assignment of all items to agents, that is, $M:I\ra A$. We will denote a class of policies by $\mathcal{C}$. Any policy $\pi$ specifies the $|I|$ turns of the agents. When an agent takes her turn, she picks her most preferred item that has not yet been allocated. We leave it to future work to consider agents
	picking strategically. Sincere picking is a reasonable starting point
	as when the policy is uncertain, a risk averse agent is likely to
	pick sincerely.


	\begin{example} Consider the setting in which $A=\{a_1,a_2\}$, $I=\{b,c,d,e\}$, the preferences of agent $a_1$ are $b\succ c\succ d\succ e$ and of agent $a_2$ are $b\succ d\succ c\succ e$. Then for the policy $a_1\rhd a_2\rhd a_2\rhd a_1$, agent $a_1$ gets $\{b,e\}$ whilst $a_2$ gets $\{c,d\}$.
		\end{example}

	We consider the following natural computational problems.



	\begin{enumerate}
	\item \possibleassignment : Given $(A,I,P,M)$ and policy class $\mathcal{C}$, does there exist a policy in $\mathcal{C}$  which results in $M$?
	\item \necessaryassignment : Given $(A,I,P,M)$, and policy class $\mathcal{C}$, is $M$ the result of all policies in $\mathcal{C}$?
	\item \possibleitem : Given $(A,I,P,a_j,o)$ where $a_j\in A$ and $o\in I$, and policy class $\mathcal{C}$, does there exist a policy in $\mathcal{C}$ such that agent $a_j$ gets item $o$?
	\item \necessaryitem : Given $(A,I,P,a_j,o)$ where $a_j\in A$ and $o\in I$, and policy class $\mathcal{C}$, does agent $a_j$ get item $o$ for all policies in $\mathcal{C}$?
	\item \possibleset : Given $(A,I,P,a_j,I')$ where $a_j\in A$ and $I'\subseteq I$, and  policy class $\mathcal{C}$, does there exist a policy in $\mathcal{C}$  such that agent $a_j$ gets exactly $I'$?
	\item \necessaryset : Given $(A,I,P,a_j,I')$ where $a_j\in A$ and $I'\subseteq I$, and policy class $\mathcal{C}$, does agent $a_j$ get exactly $I'$ for all policies in $\mathcal{C}$?
	\item \possiblesubsetallocation : Given $(A,I,P,a_j,I')$ where $a_j\in A$ and $I'\subseteq I$, and policy class $\mathcal{C}$, does there exist a policy in $\mathcal{C}$ such that agent $a_j$ gets $I'$?
			\item \necessarysubset : Given $(A,I,P,a_j,I')$ where $a_j\in A$ and $I'\subseteq I$, and policy class $\mathcal{C}$ does agent $a_j$ get $I'$ for all  policies in $\mathcal{C}$?
					\end{enumerate}

			We will consider problems  top-$k$ \possibleset and top-$k$ \necessaryset that are restrictions of  \possibleset and \necessaryset in which the set of items $I'$ is the set of top $k$ items of the distinguished agent. When policies are chosen at random, the possible and necessary allocation problems we consider are also fundamental to understand more complex problems of computing the probability of certain allocations.


	\paragraph{Contributions.}
	Our contributions are two fold. First, we provide necessary and sufficient conditions for an allocation to be the outcome of balanced policies, recursively balanced policies, and balanced alternation policies,  respectively. 
	Previously \citet{BrKi05a} characterized the outcomes of arbitrary policies. In a similar vein, 
	we provide sufficient and necessary conditions for more interesting classes of policies such as recursively balanced and balanced alternation.
	Second, we provide a detailed analysis of the computational complexity of possible and necessary allocations 
	under sequential policies. 
	Table~\ref{table:summaryseq} summarizes our complexity results. 
	Our NP/coNP-completeness results also imply that there exists no polynomial-time algorithm that can approximate within any factor the number of admissible policies which do or do not satisfy the target goals.

\newcommand{\margin}{\hspace{1.2mm}}
		\setlength\extrarowheight{2pt}
		\begin{table*}[h!]
			\centering
			\scalebox{0.75}{
			\begin{tabular}{|@{\margin}l@{\margin}|@{\margin}l@{\margin}|@{\margin}l@{\margin}|@{\margin}l@{\margin}|@{\margin}l@{\margin}|@{\margin}l@{\margin}|}
				\hline
		\multirow{ 2}{*}{{\bf Problems}}		& \multicolumn{5}{c|}{{\bf Sequential Policy Class}}\\
				\cline{2-6} 
		&Any&Balanced&Recursively Balanced&Strict Alternation& Balanced Alternation\\
		\hline
		
\possibleitem&in P&NPC (Thm.~\ref{th:saban})&NPC (Thm.~\ref{th:saban})&NPC (Thm.~\ref{th:saban})&NPC (Thm.~\ref{th:saban})\\
\hline		\necessaryitem&in P&\begin{tabular}{l}coNPC (Thm.~\ref{thm:nibunfixed});\\
		 in P for const. $k$ (Thm.~\ref{thm:nib})\end{tabular}&coNPC for all $k\ge 2$ (Thm.~\ref{thm:nirb})&coNPC for all $k\ge 2$ (Thm.~\ref{thm:knstrict})&coNPC for all $k\ge 2$ (Thm.~\ref{thm:palla})\\
		\hline
	
\possibleset&in P&NPC (Thm.~\ref{th:saban})&NPC (Thm.~\ref{th:saban})&NPC (Thm.~\ref{th:saban})&NPC (Thm.~\ref{th:saban})\\
		\hline

\necessaryset&in P&  in P  (Thm.~\ref{thm:knsb})&coNPC for all $k\ge 2$ (Thm.~\ref{thm:nirb})&coNPC for all $k\ge 2$ (Thm.~\ref{thm:knstrict})&coNPC for all $k\ge 2$ (Thm.~\ref{thm:knsa})\\
		\hline

Top-$k$ \possibleset &in P&in P (trivial)&\begin{tabular}{l}NPC for all $k\ge 3$  (Thm.~\ref{thm:kpsrb});\\
 in P for $k=2$ (Thm.~\ref{thm:topkpossibleset-rec-bal})\end{tabular}&\begin{tabular}{l}NPC for all $k\ge 3$  (Thm.~\ref{thm:pastrict});\\
 in P for $k=2$ (Thm.~\ref{thm:psk2strict})\end{tabular}&NPC for all $k\ge 2$ (Thm.~\ref{thm:palla})\\ \hline
 
		Top-$k$ \necessaryset &in P&in P  (Thm.~\ref{thm:knsb})&coNPC for all $k\ge 2$ (Thm.~\ref{thm:nirb})&coNPC for all $k\ge 2$ (Thm.~\ref{thm:knstrict})&coNPC for all $k\ge 2$ (Thm.~\ref{thm:knsa})\\
		\hline

\possiblesubsetallocation&in P&NPC (Thm.~\ref{th:saban})&NPC (Thm.~\ref{th:saban})&NPC (Thm.~\ref{th:saban})&NPC (Thm.~\ref{th:saban})\\ \hline

\necessarysubset&in P&\begin{tabular}{l}coNPC (Thm.~\ref{thm:nibunfixed});\\
 in P for const. $k$ (Thm.~\ref{thm:nallb})\end{tabular}&coNPC for all $k\ge 2$ (Thm.~\ref{thm:nirb})&coNPC for all $k\ge 2$ (Thm.~\ref{thm:knstrict})&coNPC for all $k\ge 2$ (Thm.~\ref{thm:palla})\\
		\hline

\possibleassignment&in P&in P (Coro.~\ref{coro:pab})&in P (Coro.~\ref{coro:passrb})&in P (Coro.~\ref{coro:pasa})&in P (Coro.~\ref{coro:passa})\\ \hline
			\necessaryassignment&in P&in P (Thm.~\ref{thm:nassb})&in P (Thm.~\ref{thm:nassrb})&in P (Thm.~\ref{thm:pasa})& in P (Thm.~\ref{thm:nassa})\\
		\hline
			\end{tabular}
		}

			\caption{Complexity of possible and necessary allocation for sequential allocation. 
			All possible allocation problems are NPC for $k=1$. All necessary problems are in P for $k=1$.}
			\label{table:summaryseq}
		\end{table*}

	\paragraph{Related Work.}

	Sequential allocation has been considered in the operations research and fair division literature (e.g. \citep{KoCh71a,BrTa96a}). It was popularized within the AI literature as a simple yet effective distributed mechanism~\citep{BoLa11a} and has been studied in more detail subsequently
	\citep{KNW13a,KNWX13a,BoLa14b}. In particular, the complexity of manipulating an agent's preferences  has been studied~\citep{BoLa11a,BoLa14b} 
	supposing that one agent knows the preferences of the other agents as well as the policy. Similarly in the problems we consider, the central authority knows beforehand the preferences of all agents.

	The problems considered in the paper are similar in spirit to a class of control problems studied in voting theory: if it is possible to  select a voting rule from the set of voting rules, can one be selected to obtain a certain outcome~\citep{ErEl12a}.  They are also related to a class of control problems in knockout tournaments: does there exist a draw of a tournament for which a given player wins the tournament~\citep{VAS09a,AGM+14a}. Possible and necessary winners have also been considered in voting theory for settings in which the preferences of the agents are not fully specified~\citep{Konczak05:Voting,Betzler10:Towards,Baumeister10:Taking,Bachrach10:Probabilistic,XiCo11a,ABF+12a}.

	When $n=m$, serial dictatorship is a well-known mechanism in which there is an ordering of agents and with respect to that ordering agents pick the most preferred unallocated item in their turns~\citep{Sven99a}. We note that serial dictatorship for $n=m$ is a balanced, recursively balanced and balanced alternation policy.

	\section{Characterizations of Outcomes of Sequential Allocation}
	In this section we provide necessary and sufficient conditions for a given allocation to be the outcome of a balanced policy, recursively balanced policy, or balanced alternation policy. We first define conditions on an allocation $M$. An allocation is \emph{Pareto optimal} if there is no other allocation in which each item of each agent is replaced by at least as preferred an item and at least one item of some agent is replaced by a more preferred item.

	\begin{cond}\label{cond:1}  $M$ is Pareto Optimal.
	\end{cond}

	\begin{cond}\label{cond:2}  $M$ is balanced.\end{cond}
	It is well-known that Condition~\ref{cond:1} characterizes outcomes of all sequential allocation mechanisms (without constraints). \citet{BrKi05a} proved that an assignment is achievable via sequential allocation iff it satisfies Condition~\ref{cond:1}.
	 The theorem of \citet{BrKi05a} generalized the characterization of \citet{AbSo98a} of Pareto optimal assignments as outcomes of serial dictatorships when $m=n$. 
	 We first observe the following simple adaptation of the characterization of \citet{BrKi05a} to characterize possible outcomes of balanced policies:  

	 \begin{remark}\label{thm:pab}
	 Given a profile $P$, an allocation $M$ is the outcome of a  balanced policy if and only if  $M$ satisfies Conditions~\ref{cond:1} and~\ref{cond:2}.\end{remark}

	Given a balanced allocation $M$, for each agent $a_j\in A$ and each $i\leq k$, let $p_j^i$ denote the item that is ranked at the $i$-th position by agent $a_j$ among all items allocated to agent $a_j$ by $M$. The third condition requires that for all $1\leq t<s\leq k$, no agent prefers the $s$-th ranked item allocated to any other agent to the $t$-th ranked item allocated to her.
	\begin{cond}\label{cond:3}  For all $1\leq t<s\leq k$ and all pairs of agent $a_j,a_{j'}$, agent $a_j$ prefers $p_j^t$ to $p_{j'}^s$.
	\end{cond}

	The next theorem states that Conditions 1 through 3 characterize outcomes of recursively balanced policies.


	\begin{thm}\label{thm:passrb}
	Given a profile $P$, an allocation $M$ is the outcome of a recursively balanced policy if and only if it satisfies Conditions~\ref{cond:1}, \ref{cond:2}, and \ref{cond:3}.
	\end{thm}
	\begin{proof} 
	To prove the ``only if'' direction, clearly if $M$ is the outcome of a recursively  balanced policy then Condition~\ref{cond:1}  and~\ref{cond:2} are satisfied. If Condition~\ref{cond:3} is not satisfied, then there exists $1\leq t<s\leq k$ and a pair of agents $a_j,a_{j'}$ such that agent $a_j$ prefers $p_{j'}^s$ to $p_j^t$. We note that in the round when agent $a_j$ is about to choose $p_j^t$ according to $M$, $p_{j'}^s$ is still available, because it is allocated by $M$ in a later round. However, in this case agent $a_j$ will not choose $p_j^t$ because it is not her top-ranked available item, which is a contradiction.

	To prove the ``if'' direction, for any allocation $M$ that satisfies the three conditions we will construct a recursively  balanced policy $\mo$. For each $i\leq k=m/n$, we let {\bf phase} $i$ denote the $((i-1)n+1)$-th round through $in$-th round. It follows that for all $i\leq k$, $\{p_j^i:j\leq n\}$ are allocated in phase $i$. Because of Condition~\ref{cond:3}, $\{p_j^i:j\leq n\}$ is a Pareto optimal allocation when all items in $\{p_j^{i'}:i'<i,j\leq n\}$ are removed. Therefore there exists an order $\mo_i$ over $A$ that gives this allocation. Let $\mo=\mo_1\rhd \mo_2\rhd\cdots\rhd \mo_k$. It is not hard to verify that $\mo$ is recursively balanced and $M$ is the outcome of $\mo$.\end{proof}

	Given a profile $P$ and an allocation $M$ that is the outcome of a recursively balanced policy, that is, it satisfies the three conditions as proved in Theorem~\ref{thm:passrb}, we construct a directed graph $G_M=(A,E)$, where the vertices are the agents, and we add the edges in the following way. 
	For each odd $i\leq k$, we add a directed edge $a_{j'}\ra a_j$ if and only if agent $a_j$ prefers $p_{j'}^i$ to $p_{j}^i$ and the edge is not already in $G_M$; for each even $i\leq k$, we add a directed edge $a_j\ra a_{j'}$ if and only if agent $a_j$ prefers $p_{j'}^i$ to $p_{j}^i$ and the edge is not already in $G_M$.

	\begin{cond}\label{cond:4} Suppose $M$ is the outcome of a recursively balanced policy. There is no cycle in $G_M$.
	\end{cond}

	\begin{thm}\label{thm:paa}
	An allocation $M$ is achievable by a balanced alternation policy if and only if satisfies Conditions~\ref{cond:1}, \ref{cond:2}, \ref{cond:3}, and \ref{cond:4}.\end{thm}

	\begin{proof} The ``only if'' direction: Suppose $M$ is achievable by a balanced alternation policy $\mo$. Let $\mo'$ denote the suborder of $\mo$ from round $1$ to round $n$. Let $G_{\mo'}=(A,E')$ denote the directed graph where the vertices are the agents and there is an edge $a_{j'}\ra a_j$ if and only if $a_{j'}\rhd_{\mo'} a_j$. It is easy to see that $G_{\mo'}$ is acyclic and complete. 
	We claim that $G_M$ is a subgraph of $G_{\mo'}$.  For the sake of contradiction suppose there is an edge $a_j\ra a_{j'}$ in $G_M$ but not in $G_{\mo'}$. If $a_j\ra a_{j'}$ is added to $G_M$ in an odd round $i$, then it means that agent $j'$ prefers $p_j^i$ to $p_{j'}^i$. Because $a_j\ra a_{j'}$ is not in $G_{\mo'}$, $a_{j'}\rhd_{\mo'} a_j$. This means that right before $a_{j'}$ choosing $p_{j'}^i$ in $M$, $p_j^i$ is still available, which contradicts the assumption that $a_{j'}$ chooses $p_{j'}^i$ in $M$. If $a_j\ra a_{j'}$ is added to $G_M$ in an even round, then following a similar argument we can also derive a contradiction. Therefore, $G_M$ is a subgraph of $G_{\mo'}$, which means that $G_M$ is acyclic.

	The ``if'' direction: Suppose the four conditions are satisfied. Because $G_M$ has no cycle, we can find a linear order $\mo'$ over $A$ such that $G_M$ is a subgraph of $G_{\mo'}$. We next prove that $M$ is achievable by the balanced alternation policy $\mo$ whose first $n$ rounds are $\mo'$.  For the sake of contradiction suppose this is not true and let $t$ denote the earliest round that the allocation in $\mo$ differs the allocation in $M$. Let $a_j$ denote the agent at the $t$-th round of $\mo$, let $p_{j'}^{i'}$ denote the item she gets at round $t$ in $\mo$, and let $p_j^i$ denote the item that she is supposed to get according to $M$. Due to Condition~\ref{cond:3}, $i'\leq i$. If $i'<i$ then agent $a_{j'}$ didn't get item $p_{j'}^{i'}$ in a previous round, which contradicts the selection of $t$. Therefore $i'=i$. If $i$ is odd, then there is an edge $a_{j'}\ra a_j$ in $G_M$, which means that $a_{j'}\rhd_{\mo'} a_j$. This means that $a_{j'}$ would have chosen $p_{j'}^i$ in a previous round, which is a contradiction. If $i$ is even, then a similar contradiction can be derived. Therefore $M$ is achievable by $\mo$.
	\end{proof}

	Given a profile $P$ and an allocation $M$ that is the outcome of a recursively balanced policy, that is, it satisfies the three conditions as proved in Theorem~\ref{thm:passrb}, we construct a directed graph $H_M=(A,E)$, where the vertices are the agents, and we add the edges in the following way. For each $j\leq n$ and $i\leq k$, we let $p_j^i$ denote the item that is ranked at the $i$-th position among all items allocated to agent $j$. For each $i\leq k$, if we add a directed edge $a_{j'}\ra a_j$ if $j$ prefers $p_{j'}^i$ to $p_{j}^i$ if the edge is not already there.

	\begin{cond}\label{cond:H_M} Suppose $M$ is the outcome of a recursively balanced policy. There is no cycle in $H_M$.
	\end{cond}

	\begin{thm}\label{thm:SA-charac}
	An allocation $M$ is achievable by a strict alternation policy if and only if satisfies Condition~\ref{cond:1}, \ref{cond:2}, \ref{cond:3}, and \ref{cond:H_M}.
	\end{thm}
	\begin{proof}
	The ``only if'' direction: If $M$ is an outcome of a recursively balanced policy but does not satisfy  \ref{cond:H_M}, then this means that there is a cycle in $H_M$. Let agents $a_i$ and $a_j$ be in the cycle. This means that $a_i$ is before $a_j$ in one round and $a_j$ is before $a_i$ in some other round.

	The ``if'' direction: Now assume that $M$ is an outcome of a recursively balanced policy but is not alternating. This means that there exist at least two agents $a_i$ and $a_j$ such that $a_i$ comes before $a_j$ in one round and $a_j$ comes before $a_i$ in some other round. But this means that there is cycle $a_i\rightarrow a_j \rightarrow a_i$ in graph $H_M$.
	\end{proof}

	%
	%



	\section{General Complexity Results}

	Before we delve into the complexity results, we observe the following reductions between various problems.

	\begin{lemma}\label{lem:reduction} Fixing the policy class to be one of \{all, balanced policies, recursively balanced policies,  balanced alternation policies\}, there exist polynomial-time many-one reductions between the following problems:
		\possibleset to \possiblesubsetallocation; \possibleitem to \possiblesubsetallocation; Top-$k$ \possibleset to \possibleset; \necessaryset to \necessarysubset; \necessaryitem to \necessarysubset
	; and Top-$k$ \necessaryset to \necessaryset.

	%
	%
	%
	%
	%
	\end{lemma}

	A polynomial-time many-one reduction from problem $Q$ to problem $Q'$ means that if $Q$ is NP(coNP)-hard then $Q'$ is also NP(coNP)-hard, and if $Q'$ is in P then $Q$ is also in P. We also note the following.

	\begin{remark}
		For $n=2$,	\possibleassignment and \possibleset are equivalent for any type of policies. Since $n=2$, the allocation of one agent completely determines the overall assignment.
	\end{remark}
	
	For $m=n$, checking whether there is a serial dictatorship under which each agent gets exactly one item and a designated agent $a_j$ gets item $o$ is NP-complete~\citep[Theorem 2, ][]{SaSe13a}. They also proved that for $m=n$, checking if for all serial dictatorships, 
	agent $a_j$ gets item $o$ is polynomial-time solvable. Hence, we get the following statements.

			\begin{remark}\label{th:saban}
		\possibleitem and \possibleset is NP-complete for balanced, recursively balanced as well as balanced alternation policies.
	\end{remark}


	%
	%
	%
	%
	%
	%

	\begin{remark}
		For $m=n$, \necessaryitem and \necessaryset is polynomial-time solvable for balanced, recursively balanced, and balanced alternation policies.
	\end{remark}

	Theorem~\ref{th:saban} does not necessarily hold if we consider the top element or the top $k$ elements. Therefore, we will especially consider top-$k$ \possibleset.

	\section{Arbitrary Policies}

	We first observe that for arbitrary policies, \possibleitem, \necessaryitem and \necessaryset are trivial: \possibleitem always has a yes answer (just give all the turns to that agent) and \necessaryitem and \necessaryset always have a no answer (just don't give the agent any turn). Similarly, \necessaryassignment always has a no answer.

	\begin{remark}\label{th:trivial-arbitrary}
				\possibleitem, \necessaryitem, \necessaryset, and \necessaryassignment are polynomial-time solvable for arbitrary policies.
			\end{remark}

	%
	%
	%
	%

	%
	%
	%

	\begin{thm}
		\possibleassignment is polynomial-time solvable for arbitrary policies.
	\end{thm}
	\begin{proof}
	By the characterization of \citet{BrKi05a},  all we need to do is to check whether the assignment is Pareto optimal. 
	It can be checked in polynomial time $O(|I|^2)$ whether a given assignment is Pareto optimal via an extension of a result \citet{ACMM05a}.
		\end{proof}
	
	There is also a polynomial-time algorithm for \possibleset for arbitrary policies.
	
		\begin{thm}\label{prop:getalloc-arbitrary}
		\possibleset is polynomial-time solvable for arbitrary policies.
		\end{thm}
		\begin{proof}The following algorithm works for \possibleset.
			Let the target allocation of agent $a_i$ be $S$. If there is any agent $a_j\in A\setminus \{a_i\}$ who wants to pick an item $o'\in I\setminus S$, let him pick it. If no agent in $A\setminus \{a_i\}$ wants to pick an item $o'\in I\setminus S$, and $i$ does not want to pick an item from $ S$ return no. If no agent in $A\setminus \{a_i\}$ wants to pick an item $o'\in I\setminus S$, and $i$ wants to pick an item $o \in S$, let $a_i$ pick $o$.
		If some agent in $A\setminus \{a_i\}$ wants to pick an item $o\in S$, and also $i$ wants to pick $o \in S$, then we let $a_i$ pick $o$. Repeat the process until all the items are allocated or we return no at some point. 
	\end{proof}

	\section{Balanced Policies}

	In contrast to arbitrary policies, 
	\possibleitem, \necessaryitem, \necessaryset, and \necessaryassignment are more interesting for balanced policies since we may be restricted in allocating items to a given agent to ensure balance. Before we consider them, we get the following corollary of Remark~\ref{thm:pab}.

	\begin{coro}\label{coro:pab}\possibleassignment for balanced assignments is in P.
	\end{coro}
	%

	Note that 	an assignment is achieved via all balanced policies iff
	the		assignment is the unique balanced assignment that is Pareto optimal. This is only possible if each agent gets his top $k$ items. Hence, we obtain the following.
	
		\begin{thm}\label{thm:nassb} \necessaryassignment for balanced assignments is in P.
		\end{thm}
		
			Compared to \necessaryassignment, the other `necessary' problems are more challenging.  

	\begin{thm}\label{thm:nib} For any constant $k$, \necessaryitem
	for balanced policies is in P. 
	\end{thm}
	\begin{proof} Given a \necessaryitem instance $(A,I,P,a_1,o)$, if $o$ is ranked below the $k$-th position by agent $a_1$ then we can return ``No'', because by letting agent $a_1$ choose in the first $k$ rounds she does not get item $o$.

	Suppose $o$ is ranked at the $k'$-th position by agent $a_1$ with $k'\leq k$, the next claim provides an equivalent condition to check whether the \necessaryitem instance is a ``No'' instance.
	\begin{claim}\label{claim:nib}
	Suppose $o$ is ranked at the $k'$-th position by agent $a_1$ with $k'\leq k$,  the \necessaryitem instance $(A,I,P,a_1,o)$ is a ``No'' instance if and only if there exists a balanced policy $\mo$ such that (i) agent $a_1$ picks items in the first $k'-1$ rounds and the last $k-k'+1$ rounds, and (ii) agent $a_1$ does not get $o$.
	\end{claim}
	%
	%
	Let $I^*$ denote agent $a_1$'s top $k'-1$ items. In light of Claim~\ref{claim:nib}, to check whether the $(A,I,P,a_1,o)$ is a ``No'' instance, it suffices to check for every set of $k-k'+1$ items ranked below the $k'$-th position by agent $a_1$, denoted by $I'$, whether it is possible for agent $a_1$ to get $I^*$ and $I'$ by a balanced policy where agent $a_1$ picks items in the first $k'-1$ rounds and the last $k-k'+1$ rounds. To this end, for each $I'\subseteq I-I^*-\{o\}$ with $|I'|=k-k'+1$, we construct the following maximum flow problem $F_{I'}$, which can be solved in polynomial-time by e.g.~the  Ford-Fulkerson algorithm. 
	\begin{itemize}
	\item {\bf Vertices:} $s,t$, $A-\{a_1\}$, $I-I'-I^*$.
	\item {\bf Edges and weights:} For each $a\in A-\{a_1\}$, there is an edge $s\ra a$ with weight $k$; for each $a\in A-\{a_1\}$ and $c\in I-I'-I^*$ such that agent $a$ ranks $c$ above all items in $I'$, there is an edge $a\ra c$ with weight $1$; for each $c\in I-I'-I^*$, there is an edge $c\ra t$ with weight $1$.
	\item {\bf We are asked} whether the maximum amount of flow from $s$ to $t$ is $k(n-1)$ (the maximum possible flow from $s$ to $t$).
	\end{itemize}

	\begin{claim} $(A,I,P,a_1,o)$ is a ``No'' instance if and only if there exists $I'\subseteq I-I^*-\{o\}$ with $|I'|=k-k'+1$ such that $F_{I'}$ has a solution.
	\end{claim}
	%
	Because $k$ is a constant, the number of $I'$ we will check is a constant. 
	Algorithm~\ref{alg:ni} is a polynomial algorithm for NECESSARYITEM with balanced policies.
	\end{proof}


	\begin{algorithm}[htp]
	\caption{ \necessaryitem for balanced policies.\label{alg:ni}}
	\SetAlgoLined
	\LinesNumbered
	\DontPrintSemicolon
	\footnotesize
	\Indm
	\KwIn{A \necessaryitem instance $(A,I,P,a_j,o)$.}
	\Indp
	\If {$o$ is ranked below the $k$-th position by agent $a_j$}{\Return ``No''.}
	Let $I^*$ denote agent $a_j$'s top $k'-1$ items.\;
	\For{$I'\subseteq I-I^*-\{o\}$ with $|I'|=k-k'+1$}{
	\If{$F_{|I'|}$ has a solution}{\Return ``No''}}
	\Return ``Yes''.
	\end{algorithm}

	\begin{thm}\label{thm:nallb} For any constant $k$, \necessaryset and \necessarysubset
	for balanced policies are in P.
	\end{thm}
	\begin{proof} W.l.o.g.~given a \necessaryset instance $(A,I,P,a_1,I')$, if $I'$ is not the top-ranked $k$ items of agent $a_1$ then it is a ``No'' instance because we can simply let agent $a_1$ choose items in the first $k$ rounds. When $I'$ is top-ranked $k$ items of agent $a_1$,  $(A,I,P,a_1,I')$ is a ``No'' instance if and only if $(A,I,P,a_1,o)$ is a ``No'' instance for some $o\in I'$, which can be checked in polynomial time by Theorem~\ref{thm:nib}.
	A similar algorithm works for \necessarysubset.\end{proof}
	\begin{thm}\label{thm:nibunfixed} \necessaryitem and \necessarysubset for balanced policies where $k$ is not fixed  is coNP-complete. 
	\end{thm}
	\begin{proof} Membership in coNP is obvious. By Lemma~\ref{lem:reduction} it suffices to prove that  \necessaryitem is coNP-hard, which we will prove by a reduction from \possibleitem for $k=1$, which is NP-complete \citep{SaSe13a}. Let $(A,I,P,a_1,o)$ denote an instance of the possible allocation problem for $k=1$, where $A=\{a_1,\ldots,a_n\}$, $I=\{o_1,\ldots,o_n\}$, $o\in I$, $P=(P_1,\ldots,P_n)$ is the preference profile of the $n$ agents, and we are asked whether it is possible for agent $a_1$ to get item $o$ in some sequential allocation.  Given $(A,I,P,a_1,o)$, we construct the following \necessaryitem instance.

	{\bf Agents:} $A'=A\cup \{a_{n+1}\}$.

	{\bf Items:} $I'=I\cup D\cup F_1\cup\cdots \cup F_n$, where $|D|=n-1$ and for each $a_j\in A$, $|F_j|=n-2$. We have $|I'|=(n+1)(n-1)$ and $k=n-1$.

	{\bf Preferences:} 
	\begin{itemize}
	\item The preferences of $a_1$ is $[F_1\succ P_1\succ\text{others}]$.
	\item For any $j\leq n$, the preferences of $a_j$ are obtained from $[F_j\succ P_j]$ by replacing $o$ by $D$, and then add $o$ to the bottom position.
	\item The preferences for $a_{n+1}$ is $[o\succ \text{others}]$.
	\end{itemize}

	{\bf We are asked} whether agent $a_{n+1}$ always gets item $o$.

	If $(A,I,P,a_1,o)$ has a solution $\mo$, we show that the \necessaryitem instance is a ``No'' instance by considering $\underbrace{\mo\rhd\cdots\rhd\mo}_{n-1}\rhd \underbrace{a_{n+1}\rhd \cdots\rhd a_{n+1}}_{n-1}$. In the first $(n-2)n$ rounds all $F_j$'s are allocated to agent $a_j$'s. In the following $n$ rounds $o$ is allocated to $a_1$, which means that $a_{n+1}$ does not get $o$.

	Suppose the \necessaryitem instance is a ``No'' instance and agent $n+1$ does not get $o$ in a balanced policy $\mo'$.  Because agent $a_2$ through $a_n$ rank $o$ in their bottom position, $o$ must be allocated to agent $a_1$. Clearly in the first $n-2$ times when agent $a_1$ through $a_n$ choose items, they will choose $F_1$ through $F_n$ respectively. Let $\mo$ denote the order over which agents $a_1$ through $a_n$ choose items for the last time. We obtain another order $\mo^*$ over $A$ from $\mo$ by moving all agents who choose an item in $D$ after agent $a_1$ while keeping other orders unchanged. It is not hard to see that the outcomes of running $\mo$ and $\mo^*$ are the same from the first round until agent $a_1$ gets $o$. This means that $\mo^*$ is a solution to $(A,I,P,a_1,o)$.\end{proof}

	\begin{thm}\label{thm:knsb}\necessaryset and top-$k$ \necessaryset for balanced policies are in P even when $k$ is not fixed.
	\end{thm}
	\begin{proof} Given an instance of \necessaryset, if the target set is not top-$k$ then the answer is ``No'' because we can simply let the agent choose $k$ items in the first $k$ rounds. It remains to show that top-$k$ \necessaryset for balanced policies is in P. That is, given $(A,I,P,a_1)$, we can check in polynomial time whether there is a balanced policy $\mo$ for which agent $a_1$ does not get exactly her top $k$ items. 

	For \necessaryset, suppose agent $a_1$ does not get her top-$k$ items under $\mo$. Let $\mo'$ denote the order obtained from $\mo$ by moving all agent $a_1$'s turns to the end while keeping the other orders unchanged. It is easy to see that agent $a_1$ does not get her top-$k$ items under $\mo'$ either. Therefore, \necessaryset is equivalent to checking whether there exists an order $\mo$ where agent $a_1$ picks item in the last $k$ rounds so that agent $a_1$ does not get at least one of her top-$k$ items.

	We consider an equivalent, reduced allocation instance where the agents are $\{a_1,a_2,\ldots,a_n\}$, and there are $k(n-1)+1$ items $I'=(I-I^*)\cup \{c\}$, where $I^*$ is agent $a_1$'s top-$k$ items. Agent $a_j$'s preferences over $I'$ are obtained from $P_j$ by replacing the first occurrence of items in $I^*$ by $c$, and then removing all items in $I^*$ while keeping the order of other items the same. We are asked whether there exists an order $\mo$ where agent $a_1$ is the last to pick and $a_1$ picks a single item, and each other agents picks $k$ times, so that agent $a_1$ does not get item $c$. This problem can be solved by a polynomial-time algorithm based on maximum flows that is similar to the algorithm for \necessaryitem for balanced policies in Theorem~\ref{thm:nib}.
	\end{proof}

	\section{Recursively Balanced Policies}

	In this section, we consider recursively balanced policies.
	From Theorem~\ref{thm:passrb}, we get the following corollary.

	\begin{coro}\label{coro:passrb}
	\possibleassignment for recursively balanced policies is in P.
	\end{coro}
	
	We also report computational results for problems other than \possibleassignment 

	\begin{thm}\label{thm:nassrb} \necessaryassignment for recursively balanced policies is in P.
		\end{thm}
		\begin{proof}[Proof Sketch]
		
	We initialize $t$ to $1$ i.e., focus on the first round.
	We check if there is an agent whose turn has not come in the round whose most preferred unallocated item is not $p^i_t$. In this case return ``No''. Otherwise, we complete the round in any order. 
	If all the items are allocated, we return ``Yes''.
	If $t\neq k$, we increment $t$ by one and repeat.
	\end{proof}

	The other `necessary problems' turn out to be computationally intractable.

	\begin{thm}\label{thm:nirb} For $k\geq 2$, \necessaryitem, \necessaryset, top-$k$ \necessaryset, and \necessarysubset for recursively balanced policies are coNP-complete.
	\end{thm}

	\begin{thm}\label{thm:topkpossibleset-rec-bal}
		Top-$k$ \possibleset for recursively balanced policies is in P for $k=2$.
		\end{thm}
	\begin{proof}[Proof Sketch]
		Let the agent under question be $a_1$.
	We give agent $a_1$ the first turns in each round with $s_1,s_2$ $a_1$'s top two items. 
	The agent is guaranteed to get $s_1$.
	We now construct a bipartite graph $G=((A\setminus \{a_1\}) \cup (I\setminus \{s_1\}), E)$ in which each $\{a_i,o\}\in E$
	 iff iff $a_i$ prefers $o$ to $s_2$. We check whether $G$ admits a matching that perfectly matches the agent nodes. If $G$ does not, we return no. Otherwise, there exists a recursively balanced policy for which agent $a_1$ gets $s_1$ and $s_2$. 
	 \end{proof}
	 Finally, top-$k$-\possibleset is NP-complete iff $k\geq 3$.
 
	 \begin{thm}\label{thm:kpsrb} For all $k\geq 3$, top-$k$ \possibleset
	for balanced policies  is NP-complete. 
	\end{thm}

	The proof is given in the appendix.

\section{Strict Alternation Policies} 	As for balanced alternation polices , there are $n!$ possible strict alternation policies, so if $n$ is constant, then all problems can be solved in polynomial time by brute force search. 

	\begin{thm}
	If the number of agents is constant, then \possibleitem, \possibleset, \necessaryitem, \necessaryset, \possibleassignment, and \necessaryassignment  are polynomial-time solvable for strict alternation policies.
	\end{thm}

	As a result of our characterization of strict alternation outcomes (Theorem~\ref{thm:SA-charac}), we get the following.

	\begin{coro} \label{coro:pasa}
	\possibleassignment for strict alternation polices is in P.
	\end{coro}
	
	We also present other computational results.

	\begin{thm}\label{thm:pasa}
	\necessaryassignment for strict alternation polices is in P.
	\end{thm}

	\begin{thm}\label{thm:psk2strict}
		Top-$k$ \possibleset for strict alternation policies is in P for $k=2$.
		\end{thm}
For Theorem~\ref{thm:psk2strict}, the polynomial-time algorithm is similar to the algorithm for Theorem~\ref{thm:topkpossibleset-rec-bal}. The next theorems state that the remaining problems are hard to compute. Both theorems are proved by reductions from the \possibleitem problem.

	\begin{thm}\label{thm:pastrict} For all $k\ge 3$, 
	top-$k$ \possibleset is NP-complete for strict alternation policies.
	\end{thm}

	\begin{thm}\label{thm:knstrict} For all $k\geq 2$, \necessaryitem, \necessaryset, top-$k$ \necessaryset, and \necessarysubset are coNP-complete
	for strict alternation policies. 
	\end{thm}

\section{Balanced Alternation Policies} Balanced alternation policies and strict alternation policies are the most constrained class among all policy classes we study. 
	There are $n!$ possible balanced alternation policies, so if $n$ is constant, then all problems can be solved in polynomial time by brute force search. Note that such an argument does not apply to recursively balanced policies.

	\begin{thm}
	If the number of agents is constant, then \possibleitem, \possibleset, \necessaryitem, \necessaryset, \possibleassignment, and \necessaryassignment  are polynomial-time solvable for balanced alternation policies.
	\end{thm}

	As a result of our characterization of balanced alternation outcomes (Theorem~\ref{thm:paa}), we get the following.

	\begin{coro} \label{coro:passa}
	\possibleassignment for balanced alternation polices is in P.
	\end{coro}
	
	\necessaryassignment can be solved efficiently as well:

	\begin{thm}\label{thm:nassa}
	\necessaryassignment for balanced alternation polices is in P.
	\end{thm}
	\begin{proof} We first check whether it is possible to find $\mo$ over $A$ such that after running $\mo$ there exists an agent $j$ that does not get item $p_j^1$. If so then we return ``No''. Otherwise, we remove all items in $\{p_j^1:j\leq n\}$ and check whether it is possible to find  $\mo$ over $A$ such that after running $\mo$ on the reduced instance, there exists an agent $a_j$ that does not get item $p_j^2$. If so then we return ``No''. Otherwise, we iterate until  all items are removed in which case we return ``Yes''.
	\end{proof}

	We already know that for $k=m/n=1$, top-$k$ possible and necessary problems can be solved in polynomial time. The next theorems state that for any other $k$, they are NP-complete for balanced alternation policies. Theorem~\ref{thm:palla} is proved by a reduction from the {\sc exact 3-cover} problem and Theorem~\ref{thm:knsa} is proved by a reduction from the \possibleitem problem.

	\begin{thm}\label{thm:palla} For all $k\ge 2$, 
	top-$k$ \possibleset is NP-complete, 
	\necessaryitem is coNP-complete, and \necessarysubset is coNP-complete for balanced alternation policies.
	\end{thm}

	\begin{thm}\label{thm:knsa} For all $k\geq 2$, top-$k$ \necessaryset
	for balanced alternation policies  is coNP-complete. 
	\end{thm}


	%

	%
	%

	\section{Conclusions}

	We have studied sequential allocation mechanisms
	like the course allocation mechanism at Harvard Business 
	School where the policy has not been fixed or has been
	fixed but not announced. We have characterized the allocations
	achievable with three common classes of policies: recursively balanced, strict alternation, and balanced alternation policies.
	We have also identified the computational complexity of
	identifying the possible or necessary items, set or subset
	of items to be allocated to an agent when using one of these
	three policy classes as well as the class of all policies. There are several
	interesting future directions including considering
	other common classes of policies, as well as other
	properties of the outcome like the possible or necessary welfare.

\bibliographystyle{plainnat}

\newpage

	 \appendix

	\section*{Testing Pareto optimality}

	\begin{lemma}\label{lemma:ttc}
		It can be checked in polynomial time $O(|I|^2)$ whether a given assignment is Pareto optimal. 
	\end{lemma}

	The set of assignments achieved via arbitrary policies is characterized by Pareto optimal assignments. For any given assignment setting and an assignment, the \emph{corresponding cloned setting} is one in which for each item $o$ that is owned by agent $i$, we make a copy $i_o$ of agent $i$ so that each agent copy owns exactly one item. Each copy $i_o$ has exactly the same preferences as agent $i$.
	The assignment in which copies of agents get a single item is called the \emph{cloned transformation} of the original assignment.

	\begin{claim}\label{lemma:cloned}
		An assignment is Pareto optimal 
		iff its cloned transformation is Pareto optimal for the cloned setting. 
	\end{claim}
	\begin{proof}
		If an assignment is not Pareto optimal for the cloned setting, then there exists another assignment in which each of the cloned agents get at least as preferred an item and at least one agent gets a strictly more preferred item. But if the new assignment for the cloned setting is transformed to the assignment for the original setting, then the new assignment Pareto dominates  
		 the prior assignment for the original setting.
		If an assignment is not Pareto optimal (with respect to responsive preferences) then there exists another assignment that Pareto dominates it. 
		 But this implies that the new assignment also Pareto dominates the old assignment in the cloned setting.
		\end{proof}


	We are now ready to prove Lemma~\ref{lemma:ttc}.
	\begin{proof}
		By Lemma~\ref{lemma:cloned}, the problem is equivalent to checking whether the cloned transformation of the assignment is Pareto optimal in the cloned setting. Pareto optimality of an assignment in which each agent has one item can be checked in time $O(m^2)$~\citep[see e.g.,][]{ACMM05a} where $m$ is the number of items.\footnote{The main idea is to construct a trading graph in which agent points to agent whose item he prefers more. The assignment is Pareto optimal iff the graph is acyclic.}
		 Firstly, for each item $o$ that is owned by agent $i$, we make a copy $i_o$ of agent $i$ so that each agent copy owns exactly one item. Each copy $i_o$ has exactly the same preferences as agent $i$. Based on the ownership information of each the $m$ agent copies, and the preferences of the agent copies, we construct a \emph{trading graph} in which each copy $i_o$ points to each of the items more preferred than $o$. Also each $o$ points to its owner $i_o$. Then the assignment in the cloned transformation is Pareto optimal iff the trading graph is acyclic~\citep[see e.g.,]{ACMM05a}. Acyclicity of a graph can be checked in time linear in the size of the graph via depth-first search.
	\end{proof}

	\section*{Proof of Theorem~\ref{prop:getalloc-arbitrary}}
		\begin{proof}
			Let the target allocation of agent $a_i$ be $S$. If there is any agent $a_j\in A\setminus \{a_i\}$ who wants to pick an item $o'\in I\setminus S$, let him pick it. If no agent in $A\setminus \{i\}$ wants to pick such an item $o'\in I\setminus S$, and $i$ does not want to pick an item from $ S$ return no. If no agent in $A\setminus \{a_i\}$ wants to pick such an item $o'\in I\setminus S$, and $a_i$ wants to pick an item $o \in S$, let $a_i$ pick $o$.
		If some agents in $A\setminus \{a_i\}$ wants to pick such an item $o\in S$, and also $i$ wants to pick $o \in S$, then we let $a_i$ pick $o$. Repeat the process until all the items are allocated or we return no at some point. 

		We now argue for the correctness of the algorithm. 
		Observe the order in which agent $a_1$ picks items in $S$ is exactly according to his preferences.

		\begin{claim}\label{claim:o}
		Let us consider the first pick in the algorithm. If agent $a_1$ picks an item $o=\max_{\pref_i}(S)$, then if there exists a policy $\pi$ in which 
		agent $a_i$ gets $S$, then there also exists a policy $\pi'$ in which agent $a_1$ first picks $o$ and agent $i$ gets $S$ overall. 
		\end{claim}
		\begin{proof}
		In $\pi$, by the time agent $a_i$ picks his second most preferred item from $S$, all items more preferred have already been allocated. In $\pi$, if $a_i\neq \pi(1)$,  then we can obtain $\pi'$ by bringing $a_i$ to the first place and having all the other turns in the same order. Note that in $\pi'$, for any agent's turn the set of available items are either the same or $o$ is the extra item missing. However since $o$ was not even chosen by the latter agents, the picking outcomes of $\pi$ and $\pi'$ are identical. 
		\end{proof}

		%
		%

		\begin{claim}\label{claim:o'}
		Let us consider the first pick in the algorithm. If some agent $a_j$ picks an item $o'\in A\setminus S$ in the algorithm, then if there exists a policy in which agent $a_i$ gets $S$, then there also exists a policy in which agent $a_j$ first picks $o'$ and agent $a_i$ gets $S$ overall. 
		\end{claim}
		\begin{proof}
		In $\pi$, if $a_j\neq \pi(1)$,  then we can obtain $\pi'$ by bringing $a_j$ to the first place and having all the other turns in the same order. If $j$ does not get $o'$ in $\pi$, then when we construct $\pi'$ we simply delete the turn of the agent who got $o'$.
		Note that in $\pi'$, for any agent's turn the set of available items are either the same or $o'$ is the extra item missing. However since $o'$ was not even chosen by the latter agents, the picking outcomes of $\pi$ and $\pi'$ are identical. 
		\end{proof}

		By inductively applying Claims~\ref{claim:o} and \ref{claim:o'}, we know that as long as a policy exists in which $i$ gets allocation $S$, our algorithm can construct a policy in which $i$ gets allocation $S$.
		\end{proof}

	\section*{Proof of Theorem~\ref{thm:nib}}

	\begin{proof} In a \necessaryitem instance we can assume the distinguished agent is $a_1$. Given  $(A,I,P,a_1,o)$, if $o$ is ranked below the $k$-th position by agent $a_1$ then it we can return ``No'', because by letting agent $a_1$ choose in the first $k$ rounds she does not get item $o$.

	Suppose $o$ is ranked at the $k'$-th position by agent $a_1$ with $k'\leq k$, the next claim provides an equivalence condition to check whether the \necessaryitem instance is a ``No'' instance.
	\begin{claim}\label{claim:nib}
	Suppose $o$ is ranked at the $k'$-th position by agent $a_1$ with $k'\leq k$,  the \necessaryitem instance $(A,I,P,a_1,o)$ is a ``No'' instance if and only if there exists a balanced policy $\mo$ such that (i) agent $a_1$ picks items in the first $k'-1$ rounds and the last $k-k'+1$ rounds, and (ii) agent $a_1$ does not get $o$.
	\end{claim}
	\begin{proof}
	Suppose there exists a balanced policy $\mo'$ such that agent $a_1$ does not get item $o$, then we obtain $\mo^*$ from $\mo'$ by moving the first $k'-1$ occurrences of agent $a_1$ to the beginning of the sequence while keeping other positions unchanged. When preforming $\mo^*$, in the first $k'-1$ rounds agent $a_1$ gets her top $k'-1$ items. 

	By the next time agent $a_1$ picks an item in $\mo^*$, $o$ must have been chosen by another agent. To see why this is true, for each agent from the $k'$-th round until agent $a_1$'s next turn in $\mo^*$, we compare side by side the items allocated before this agent's turn by $\mo^*$ and by $\mo'$. It is not hard to see by induction that the item allocated by $\mo^*$ before agent $a_1$'s next turn is a superset of the item allocated by $\mo'$ before agent $a_1$'s $k'$-th turn. Because the latter contains $o$, agent $a_1$ does not get $o$ in $\mo^*$.

	Then, we obtain $\mo$ from $\mo^*$ by moving the $k'$-th through the $k$-th occurrence of agent $a_1$ to the end of the sequence while keeping other positions unchanged. It is easy to see that agent $a_1$ does not get $o$ in $\mo$. This completes the proof.
	\end{proof}
	Let $I^*$ denote agent $a_1$'s top $k'-1$ items. In light of Claim~\ref{claim:nib}, to check whether the $(A,I,P,a_1,o)$ is a ``No'' instance, it suffices to check for every set of $k-k'+1$ items ranked below the $k'$-th position by agent $a_1$, denoted by $I'$, whether it is possible for agent $a_1$ to get $I^*$ and $I'$ by a balanced policy where agent $a_1$ picks items in the first $k'-1$ rounds and the last $k-k'+1$ rounds. To this end, for each $I'\subseteq I-I^*-\{o\}$ with $|I'|=k-k'+1$, we construct the following maximum flow problem $F_{I'}$, which can be solved in polynomial-time by e.g.~the  Ford-Fulkerson algorithm. 
	\begin{itemize}
	\item {\bf Vertices:} $s,t$, $A-\{a_1\}$, $I-I'-I^*$.
	\item {\bf Edges and weights:} For each $a\in A-\{a_1\}$, there is an edge $s\ra a$ with weight $k$; for each $a\in A-\{a_1\}$ and $c\in I-I'-I^*$ such that agent $a$ ranks $c$ above all items in $I'$, there is an edge $a\ra c$ with weight $1$; for each $c\in I-I'-I^*$, there is an edge $c\ra t$ with weight $1$.
	\item {\bf We are asked} whether the maximum amount of flow $s$ to $t$ is $k(n-1)$ (the maximum possible flow from $s$ to $t$).
	\end{itemize}

	\begin{claim} $(A,I,P,o)$ is a ``No'' instance if and only if there exists $I'\subseteq I-I^*-\{o\}$ with $|I'|=k-k'+1$ such that $F_{I'}$ has a solution.
	\end{claim}
	\begin{proof} If $(A,I,P,o)$ is a ``No'' instance, then by Claim~\ref{claim:nib} there exists $\mo$ such that agent $a_1$ picks items in the first $k'-1$ rounds and the last $k-k'+1$ rounds, and  agent $a_1$ gets $I^*\cup I'$ for some $I'\subseteq I-I^*-\{o\}$. For each agent $a_j$ with $j\neq 2$, let there be a flow of amount $k$ from $s$ to $a_j$ and a flow of amount $1$ from $a_j$ to all items that are allocated to her in $\mo$. Moreover, let there be a flow of amount $1$ from any $c\in I-I^*-\{o\}$ to $t$. It is easy to check that the amount of flow is $k(n-1)$.

	If $F_{I'}$ has a solution, then there exists an integer solution because all weights are integers. This means that there exists an assignment of all items in $I-I'-I^*$ to agent $2$ through $n$ such that no agent gets an item that is ranked below any item in $I^*$. Starting from this allocation, after implementing all trading cycles we obtain a Pareto optimal allocation where $I-I'-I^*$ are allocated to agent $2$ through $n$, and still no agent gets an item that is ranked below any item in $I^*$. By Proposition~1 in Brams and King, there exists a balanced policy $\mo^*$ that gives this allocation. It follows that agent $a_1$ does not get $o$ under the balanced policy $\mo=\underbrace{a_1\rhd \ldots \rhd a_1}_{k'-1}\rhd \mo^*\rhd \underbrace{a_1\rhd \ldots \rhd a_1}_{k-k'+1}$.
	\end{proof}
	Because $k$ is a constant, the number of $I'$ we will check is a constant. The polynomial algorithm for \necessaryitem for balanced policies is presented as Algorithm~\ref{alg:ni}.
	\end{proof}

	\section*{Proof of Theorem~\ref{thm:nassrb}}
	%
		\begin{proof}
	In the allocation $p$, let $p^j_i$ be the $j$-th most preferred item for agent $i$ among his set of $k$ allocated items.

	\begin{claim}
	If there exists a recursively balanced policy achieving the target allocation. Then, in any such recursively balanced policy, we know that in each $t$-th round, each agent gets item $p^t_i$. 
	\end{claim}

	We initialize $t$ to $1$ i.e., focus on the first round.
	We check if there is an agent whose turn has not come in the round whose most preferred unallocated item is not $p^t_i$. In this case return ``no''. Otherwise, we complete the round in any arbitrary order. 
	If all the items are allocated, we return ``yes''.
	If $t\neq k$, we increment $t$ by one and repeat the process.

	We now argue for correctness. If the algorithm returns no, 
	then we know that there is a recursively balanced policy that does not achieve the allocation. This policy was partially built during the algorithm and can be completed in an arbitrary way to get an allocation that is not the same as the target allocation.
	Now assume for contradiction that there is a policy which does not achieve the allocation but the algorithm incorrectly returns yes. Consider the first round where the algorithm makes a mistake. 
	But in each round, each agent had a unique and mutually exclusive most preferred unallocated item. Hence no matter which policy we implement in the round, the allocation and the set of unallocated items after the round stays the same. Hence a contradiction.
	\end{proof}

\section*{Proof of Theorem~\ref{thm:nirb}}

	\begin{proof}[Proof Sketch] Membership in coNP is obvious. By Lemma~\ref{lem:reduction} it suffices to show coNP-hardness for \necessaryitem and top-$k$ \necessaryset. We will prove the co-NP-hardness for them for $k=2$ by the same reduction from \possibleitem for $k=1$, which is NP-complete~\citep{SaSe13a}. The proof for other $k\geq 2$ can be done similarly  by constructing preferences so that the distinguished agent always get her top $k-2$ items. Let $(A,I,P,a_1,o)$ denote an instance of \possibleitem for $k=1$, where $A=\{a_1,\ldots,a_n\}$, $I=\{o_1,\ldots,o_n\}$, $o\in I$, $P=(P_1,\ldots,P_n)$ is the preference profile of the $n$ agents, and we are asked wether it is possible for agent $a_1$ to get item $o$ in some sequential allocation.  Given $(A,I,P,a_1,o)$, we construct the following necessary allocation instance.

	{\bf Agents:} $A'=A\cup \{a_{n+1}\}$.

	{\bf Items:} $I'=I\cup \{c,d\}\cup D$, where $|D|=n+1$.

	{\bf Preferences:} 
	\begin{itemize}
	\item The preferences of $a_1$ is obtained from $P_1$ by inserting $d$ right before $o$, and append the other items such that the bottom item is $c$.
	\item For any $2\leq j\leq n$, the preferences of $a_j$ is obtained from $P_j$ by replacing $o$ by $D$ and then appending the remaining items such that the bottom items are $c\succ d\succ  o$.
	\item The preferences for $a_{n+1}$ is $[c\succ o\succ \text{others}\succ d]$.
	\end{itemize}

	For \necessaryitem, we are asked whether agent $a_{n+1}$ always get item $o$; for top-$k$ \necessaryset, we are asked whether agent $a_{n+1}$ always get $\{c,o\}$, which are her top-2 items.
 
	Suppose the $(A,I,P,a_1,o)$ has a solution, denoted by $\mo$. We claim that $\mo'=a_{n+1}\rhd \mo\rhd a_1\rhd (A'-\{a_1\})$ is a ``No'' answer to the \necessaryitem and top-$k$ \necessaryset instance. Following $\mo'$, in the first round $a_{n+1}$ gets $c$. In the next $n$ rounds $a_1$ gets $d$. Then in the $(n+2)$-th round agent $a_1$ gets item $o$, which means that $a_{n+1}$ does not get item $o$ after all items are allocated.

	We note that $a_{n+1}$ always get item $c$ for any recursively balanced policy. We next show that if \necessaryitem or top-$k$ \necessaryset instance is a ``No'' instance, then the \possibleitem instance is a ``Yes'' instance. Suppose $\mo'$ is a recursively balanced policy such that $a_{n+1}$ does not get $o$. We let {\bf phase 1} denote the first $n+1$ rounds, and let {\bf phase 2} denote the $(n+2)$-th through $2(n+1)$-th round.

	Because $o$ is the least preferred item for all agents except $a_1$ and $a_{n+1}$, if $a_{n+1}$ does not get $o$ in the second phase, then $o$ must be allocated to $a_1$. This is because for the sake of contradiction suppose $o$ is allocated to agent $a_j$ with $j\neq 1,n$, then $a_j$ must be the last agent in $\mo'$ and $o$ is not chosen in any previous round. However, when it is $a_n$'s turn in the second phase, $o$ is still available, which means that $a_n$ would have chosen $o$ and contradicts the assumption that $a_j$ gets $o$.

	\begin{claim} If $a_1$ gets $o$ under $\mo'$, then $a_1$ gets $d$ in the first phase.
	\end{claim}
	\begin{proof} For the sake of contradiction, suppose in the first phase $a_1$ does not get $d$, then either she gets an item before $d$, or she gets $o$, because it is impossible for $a_1$ to get an item after $o$ otherwise another agent must get $o$ in the first phase, which is impossible as we just argued above.
	\begin{itemize}
	\item  If $a_1$ gets an item before $d$ in the first phase, then in order for $a_1$ to get $o$ in the second phase, $d$ must be chosen by another agent. Clearly $d$ cannot be chosen by $a_{n+1}$ before $a_1$ gets $o$, because $d$ is the bottom item by $a_{n+1}$, which means that the only possibility for $a_{n+1}$ to get $d$ is that $a_{n+1}$ is the last agent in $\mo'$. If $d$ is chosen by $a_j$ with $j\leq n$, then because $d,o$ are the bottom two items by $a_j$, the last two agents in $\mo'$ must be $a_j\rhd a_1$ . Therefore, when $a_{n+1}$ chooses an item in the second phase, $o$ is still available, which means that $a_{n+1}$ gets $o$ in $\mo'$, a contradiction to the assumption that $a_{n+1}$ does not get her top-$2$ items.
	\item If $a_1$ gets $o$ in the first phase, then it means that another agent must get $d$ in the first phase, which is impossible because all other agents rank $d$ within their bottom two positions, which means that the earliest round that any of them can get $d$ is $2n+1$.
	\end{itemize}
	\end{proof}
	Let $\mo$ denote the order over $A$ that is obtained from the first phase of $\mo'$ by removing $a_{n+1}$, and them moving all agents who get an item in $D$ after $a_1$. We claim that $\mo$ is a solution to $(A,I,P,a_1,o)$, because when it is $a_1$'s round all items before $o$ must be chosen and $o$ has not been chosen (if another agent gets $o$ before $a_1$ in $\mo$ then the same agent must get an item in $D$ in the first phase of $\mo'$, which contradicts the construction of $\mo$). This proves the co-NP-completeness of the allocation problems mentioned in the theorem.\end{proof}

	\section*{Proof of Theorem~\ref{thm:topkpossibleset-rec-bal}}

	\begin{proof}
	We give agent $a_1$ the first turns in each round. 
	He is guaranteed to get $s_1$.
	We now construct a bipartite graph $G=((A\setminus \{a_1\}) \cup (I\setminus \{s_1\}), E)$ in which each $\{i,o\}\in E$
	 iff $o$ is strictly more preferred for $i$ than $s_2$. We check whether $G$ admits a perfect matching. If $G$ does not admit a perfect matching, we return no. Otherwise, there exists a recursively balanced policy for which agent $a_1$ gets $s_1$ and $s_2$. 
 
	 \begin{claim}
	$G$ admits a perfect matching if and only if there a recursively balanced policy for which $a_1$ gets $\{s_1,s_2\}$.
	 \end{claim}
	 \begin{proof}
 
	If $G$ admits a perfect matching, then each agent in 
	$A\setminus \{a_1\}$ can get a more preferred item than $s_2$ in the first round. If this particular allocation is not Pareto optimal for agents in  $A\setminus \{a_1\}$ for items among $I\setminus \{s_1\}$, we can easily compute a Pareto optimal Pareto improvement over this allocation by implementing trading cycles as in setting of house allocation with existing tenants. This takes at most $O(n^3)$. Hence, we can compute an allocation in which each agent in $A\setminus \{a_1\}$ gets a strictly more preferred item than $s_2$ and this allocation for agents in $A\setminus \{a_1\}$ is Pareto optimal.
	Since the allocation is Pareto optimal, we can easily build up a policy which achieves this Pareto optimal allocation via the characterization of Brams. In the second round, $a_1$ gets $s_2$ and then subsequently we don't care who gets what because agent $a_1$ has already got $s_1$ and $s_2$.
 
	 If $G$ does not admit a perfect matching, then there is no allocation in which each agent in $A\setminus \{a_1\}$ get a strictly better item than $s_2$ in $I\setminus \{s_1\}$. Hence in each policy in the first round, some agent in $A\setminus \{a_1\}$ will get $s_2$.
	 \end{proof}
	\end{proof}

\section*{Proof of Theorem~\ref{thm:kpsrb}}
	\begin{proof} Membership in NP is obvious. We prove that  top-$k$ \possibleset for $k=3$ is NP-hard by a reduction from \possibleitem for $k=1$, which is NP-complete~\citep{SaSe13a}. Hardness for other $k$'s can be proved similarly  by constructing preferences so that the distinguished agent always get her top $k-2$ items. Let $(A,I,P,a_1,o)$ denote an instance of \possibleitem for $k=1$, where $A=\{a_1,\ldots,a_n\}$, $I=\{o_1,\ldots,o_n\}$, $o\in I$, $P=(P_1,\ldots,P_n)$ is the preference profile of the $n$ agents, and we are asked wether it is possible for agent $a_1$ to get item $o$ in some sequential allocation. Given $(A,I,P,a_1,o)$, we construct the following \possibleset instance.

	{\bf Agents:} $A'=A\cup \{a_{n+1}\}\cup \{d_1,\ldots,d_{n-1}\}$.

	{\bf Items:} $I'=I\cup \{c_1,c_2,c_3\}\cup D\cup E\cup F$, where $|D|=|E|=n-1$ and $|F|=3n-1$. We have $|I'|=6n$.

	{\bf Preferences:} 
	\begin{itemize}
	\item The preferences of $a_1$ is $[P_1\succ\text{others}\succ c_1\succ c_2\succ c_3]$.
	\item For any $2\leq j\leq n$, the preferences of $a_j$ is obtained from $[P_j\succ\text{others}\succ c_1\succ c_2\succ c_3\succ E]$ by switching $o$ and $E$.
	\item The preferences for $a_{n+1}$ is $[c_1\succ c_2\succ c_3\succ \text{others}]$.
	\item For all $j\leq n-1$,  the preferences for $d_{j}$ is $[D\succ ((I-\{o\})\cup E)\succ c_3\succ c_2\succ c_1\succ \text{others}]$.
	\end{itemize}

	{\bf We are asked} whether agent $a_{n+1}$ can get items $\{c_1,c_2,c_3\}$, which are her top-$3$ items.

	If $(A,I,P,a_1,o)$ has a solution $\mo$, we show that the top-$3$ \possibleset instance is a ``Yes'' instance by considering $\mo'=\underbrace{a_{n+1}\rhd d_1\rhd \cdots \rhd d_{n-1}\rhd \mo}_{\text{Phase }1}\rhd \underbrace{a_{n+1}\rhd d_1\rhd \cdots \rhd d_{n-1}\rhd A}_{\text{Phase }2}\rhd$ $\underbrace{a_{n+1}\rhd \text{others}}_{\text{Phase }3}$. In the first phase $a_{n+1}$ gets $c_1$; $d_j$'s get $D$ $a_1$ gets $o$ and  other agents in $A$ get $n-1$ items in $(I-\{o\})\cup E$. In the second phase  $a_{n+1}$ gets $c_2$; $d_j$'s get the remaining $n-1$ items in $(I-\{o\})\cup E$; agents in $A$ get $n$ items in $F$. In the third phase $a_{n+1}$ gets $c_3$.

	Suppose the top-$3$ \possibleset instance is a ``Yes'' instance and agent $a_{n+1}$ gets $\{c_1,c_2,c_3\}$ in a recursively balanced policy $\mo'$.  Let $\mo$ denote the order over which agents $a_1$ through $n$ choose items in the first phase of $\mo'$. We obtain another order $\mo^*$ over $A$ from $\mo$ by moving all agents who choose an item in $D$ after agent $a_1$ without changing the order of other agents. We claim that $\mo^*$ is a solution to $(A,I,P,a_1,o)$. For the sake of contradiction suppose $\mo^*$ is not a solution to $(A,I,P,a_1,o)$. It follows that in the first phase of $\mo'$ agent $a_1$ gets an item she ranks higher than $o$, because no other agents can get $o$. This means that in the first phase $n$ items in $(I-\{o\})\cup E$ are chosen by $A$. We note that in the first phase $d_j$'s must chose items in $D$. Then in the second phase at least one $d_j$ will choose $\{c_3\}$, because there are $n-1$ of them and only $2(n-1)-n=n-2$ items available before $\{c_3\}$. This contradicts the assumption that $a_{n+1}$ gets $c_3$.\end{proof}

	\section*{Proof of Theorem~\ref{thm:pasa}}
			\begin{proof}
				We prove that an assignment $M$ is the outcome of all strict alternating policies iff in each round, each agent has a unique most preferred item from among the unallocated items from the previous round. If in each round, each agent gets the most preferred item from among the unallocated items from the previous round, the order does not matter in any round. Hence all alternating policies result in $M$.

	Now assume that it is not the case that in each round, each agent gets the most preferred item from among the unallocated items from the previous round. Then, there exist at least two agent who have the same most preferred item from among the remaining items. Therefore, a different relative order among such agents results in different allocations which means that $M$ is not the unique outcome of all  strict alternating policies.
				\end{proof}

\section*{Proof of Theorem~\ref{thm:pastrict}}
	\begin{proof} Membership in NP is obvious. We prove that  top-$k$ \possibleset for $k=3$ is NP-hard by a reduction from \possibleitem for $k=1$, which is NP-complete~\citep{SaSe13a}. The reduction is similar to the proof of Theorem~\ref{thm:kpsrb}. Hardness for other $k$'s can be proved similarly  by constructing preferences so that the distinguished agent always get her top $k-2$ items. Let $(A,I,P,a_1,o)$ denote an instance of \possibleitem for $k=1$, where $A=\{a_1,\ldots,a_n\}$, $I=\{o_1,\ldots,o_n\}$, $o\in I$, $P=(P_1,\ldots,P_n)$ is the preference profile of the $n$ agents, and we are asked wether it is possible for agent $a_1$ to get item $o$ in some sequential allocation. Given $(A,I,P,a_1,o)$, we construct the following \possibleset instance.

	{\bf Agents:} $A'=A\cup \{a_{n+1}\}\cup \{d_1,\ldots,d_{n-1}\}$.

	{\bf Items:} $I'=I\cup \{c_1,c_2,c_3\}\cup D\cup E\cup F$, where $|D|=|E|=n-1$ and $|F|=3n-1$. We have $|I'|=6n$.

	{\bf Preferences:} 
	\begin{itemize}
	\item The preferences of $a_1$ is $[P_1\succ\text{others}\succ c_1\succ c_2\succ c_3]$.
	\item For any $2\leq j\leq n$, the preferences of $a_j$ is obtained from $[P_j\succ\text{others}\succ c_1\succ c_2\succ c_3\succ E]$ by switching $o$ and $E$.
	\item The preferences for $a_{n+1}$ is $[c_1\succ c_2\succ c_3\succ \text{others}]$.
	\item For all $j\leq n-1$,  the preferences for $d_{j}$ is $[D\succ ((I-\{o\})\cup E)\succ c_3\succ c_2\succ c_1\succ \text{others}]$.
	\end{itemize}

	{\bf We are asked} whether agent $a_{n+1}$ can get items $\{c_1,c_2,c_3\}$, which are her top-$3$ items.

	If $(A,I,P,a_1,o)$ has a solution $\mo$, we show that the top-$3$ \possibleset instance is a ``Yes'' instance by considering $\mo'=\underbrace{a_{n+1}\rhd d_1\rhd \cdots \rhd d_{n-1}\rhd \mo}_{\text{Phase }1}\rhd \underbrace{a_{n+1}\rhd d_1\rhd \cdots \rhd d_{n-1}\rhd \mo}_{\text{Phase }2}\rhd$ $\underbrace{a_{n+1}\rhd d_1\rhd \cdots \rhd d_{n-1}\rhd \mo}_{\text{Phase }3}$. In the first phase $a_{n+1}$ gets $c_1$, $a_1$ gets $o$;  other agents in $A$ get $n-1$ items in $(I-\{o\})\cup E$; $d_j$'s get $D$. In the second phase  $a_{n+1}$ gets $c_2$; $d_j$'s get the remaining $n-1$ items in $(I-\{o\})\cup E$; agents in $A$ get $n$ items in $F$. In the third phase $a_{n+1}$ gets $c_3$.

	Suppose the top-$3$ \possibleset instance is a ``Yes'' instance and agent $a_{n+1}$ gets $\{c_1,c_2,c_3\}$ in a strict alternation policy $\mo'$.  Let $\mo$ denote the order over which agents $a_1$ through $n$ choose items in the first phase of $\mo'$. We obtain another order $\mo^*$ over $A$ from $\mo$ by moving all agents who choose an item in $D$ after agent $a_1$ without changing the order of other agents. We claim that $\mo^*$ is a solution to $(A,I,P,a_1,o)$. For the sake of contradiction suppose $\mo^*$ is not a solution to $(A,I,P,a_1,o)$. It follows that in the first phase of $\mo'$ agent $a_1$ gets an item she ranks higher than $o$, because no other agents can get $o$. This means that in the first phase $n$ items in $(I-\{o\})\cup E$ are chosen by $A$. We note that in the first phase $d_j$'s must chose items in $D$. Then in the second phase at least one $d_j$ will choose $\{c_3\}$, because there are $n-1$ of them and only $2(n-1)-n=n-2$ items available before $\{c_3\}$. This contradicts the assumption that $a_{n+1}$ gets $c_3$.\end{proof}

\section*{Proof of Theorem~\ref{thm:knstrict}}

	\begin{proof}[Proof Sketch] The proof is similar to the proof of Theorem~\ref{thm:nirb}. Membership in coNP is obvious. By Lemma~\ref{lem:reduction} it suffices to show coNP-hardness for \necessaryitem and top-$k$ \necessaryset. We will prove the co-NP-hardness for them for $k=2$ by the same reduction from \possibleitem for $k=1$, which is NP-complete~\citep{SaSe13a}. The proof for other $k\geq 2$ can be done similarly  by constructing preferences so that the distinguished agent always get her top $k-2$ items. Let $(A,I,P,a_1,o)$ denote an instance of \possibleitem for $k=1$, where $A=\{a_1,\ldots,a_n\}$, $I=\{o_1,\ldots,o_n\}$, $o\in I$, $P=(P_1,\ldots,P_n)$ is the preference profile of the $n$ agents, and we are asked wether it is possible for agent $a_1$ to get item $o$ by some strict alternation policy.  Given $(A,I,P,a_1,o)$, we construct the following necessary allocation instance.

	{\bf Agents:} $A'=A\cup \{a_{n+1}\}$.

	{\bf Items:} $I'=I\cup \{c,d\}\cup D$, where $|D|=n+1$.

	{\bf Preferences:} 
	\begin{itemize}
	\item The preferences of $a_1$ is obtained from $P_1$ by inserting $d$ right before $o$, and append the other items such that the bottom item is $c$.
	\item For any $2\leq j\leq n$, the preferences of $a_j$ is obtained from $P_j$ by replacing $o$ by $D$ and then appending the remaining items such that the bottom items are $c\succ d\succ  o$.
	\item The preferences for $a_{n+1}$ is $[c\succ o\succ \text{others}\succ d]$.
	\end{itemize}

	For \necessaryitem, we are asked whether agent $a_{n+1}$ always get item $o$; for top-$k$ \necessaryset, we are asked whether agent $a_{n+1}$ always get $\{c,o\}$, which are her top-2 items.
 
	Suppose the $(A,I,P,a_1,o)$ has a solution, denoted by $\mo$. We claim that $\mo'=\underbrace{\mo\rhd a_{n+1}}_{\text{Phase 1}}\rhd  \underbrace{\mo\rhd a_{n+1}}_{\text{Phase 2}}$ is a ``No'' answer to the \necessaryitem and top-$k$ \necessaryset instance. Following $\mo'$, in phase $a_1$ gets $d$ gets $d$ and $a_{n+1}$ gets $c$. In phase 2 $a_1$ gets $o$, which means that $a_{n+1}$ does not get item $o$ after all items are allocated.

We next show that if \necessaryitem or top-$k$ \necessaryset instance is a ``No'' instance, then the \possibleitem instance is a ``Yes'' instance.  We note that $a_{n+1}$ always get item $c$ in the first phase of any strict alternation policy. Let $\mo'$ denote a strict alternation policy where $a_{n+1}$ does not get $o$. If $a_1$ does not get $d$ in the first phase, then following a similar argument in the proof of Theorem~\ref{thm:nirb}, we have that $a_{n+1}$ gets $o$ in the second phase, which is a contradiction. Therefore, $a_1$ must get $d$ in the first phase. 

Let $\mo$ denote the order over $A$ that is obtained from the first phase of $\mo'$ by removing $a_{n+1}$, and them moving all agents who get an item in $D$ after $a_1$. We claim that $\mo$ is a solution to $(A,I,P,a_1,o)$, because when it is $a_1$'s round all items before $o$ must be chosen and $o$ has not been chosen (if another agent gets $o$ before $a_1$ in $\mo$ then the same agent must get an item in $D$ in the first phase of $\mo'$, which contradicts the construction of $\mo$). This proves the co-NP-completeness of the allocation problems mentioned in the theorem.\end{proof}

	\section*{Proof of Theorem~\ref{thm:palla}}
	\begin{proof} Membership in NP and coNP are obvious. By Lemma~\ref{lem:reduction}, if \necessaryitem is coNP-hard then \necessarysubset is coNP-hard. We show the NP-hardness of top-$k$ \possibleset and coNP-hardness of \necessaryitem by the same reduction from {\sc exact 3-cover (X3C)} for $k=2$. Hardness for other $k$ can be proved similarly by constructing preferences so that the distinguished agent always get her top $k-2$ items. In an {\sc X3C} instance $(\ms,X)$, we are given $\ms=\{S_1,\ldots,S_t\}$ and $X=\{x_1,\ldots, x_q\}$, such that $q$ is a multiple of $3$ and for all $j\leq t$, $|S_j|=3$ and $S_j\subseteq X$; we are asked whether there exists a subset of $q/3$ elements of $\ms$ whose union is exactly $X$.

	Given an {\sc X3C} instance $(\ms,X)$, we construct the following agents, items, and preferences.

	{\bf Agents}: $A=\{a\}\cup \bigcup_{j\leq t} \ms_j\cup X\cup C$, where $C=\{c_1,\ldots, c_{q/3}\}$ and $\ms_j=\{S_j,S_j^{j_1},S_j^{j_2},S_j^{j_1}\}$ such that  $j\leq t$, $j_1,j_2,j_3$ are the indices of elements $S_j$. That is, $S_j=\{x_{j_1}, x_{j_2}, x_{j_3}\}$. We note that $|A|=4t+4q/3+1$.

	{\bf Items}: $8t+8q/3+2$ items are defined as follows. Let $I=\{a,b,c\}\cup \bigcup_{j\leq t}\ms_j\cup D\cup E\cup F$, where $|D|=8q/3$, $E=q/3$, and $F=4t-q/3-1$. We note that $|I|=2 |A|$. For each $i\leq q$, we let $K_i$ denote the sets in $\ms$ that cover $x_i$. That is, $K_i=\{S\in\ms:x_i\in S\}$.

	{\bf Preferences} are illustrated in Table~\ref{tab:pref}.
	\begin{table}[htp]
	\centering
	\begin{tabular}{|r|l|}
	\hline Agent & Preferences\\
	\hline$a$:& $a\succ b\succ c\succ \text{others}$\\
	\hline $\forall j, S_j$:& $S_j\succ a\succ D\succ b\succ \text{others}\succ c$\\
	\hline $\forall j, s=1,2,3, S_j^{j_s}$:&$S_j\succ S_j^{j_s}\succ a\succ D\succ b\succ \text{others}\succ c$\\
	\hline $\forall i$, $x_i$: & $K_i\succ b\succ \text{others}\succ c$\\
	\hline $\forall k\leq q/3$, $c_k$: & $a\succ S_1\succ \ldots\succ  S_t\succ E\succ \text{others}\succ c$\\ \hline
	\end{tabular}
	\caption{Agents' preferences, where $K_i=\{S\in\ms:x_i\in S\}$.\label{tab:pref}}
	\end{table}

	For top-$2$ \possibleset, we are asked whether agent $a$ can get $\{a,b\}$. For \necessaryitem, we are asked whether agent $a$ always get item $c$.

	If the {\sc X3C} instance has a solution, w.l.o.g.~$\{S_1,\ldots,S_{q/3}\}$, we show that there exists a solution to the constructive control problem and destructive control problem described above. For each $j\leq t$, we let $L_j=S_j\rhd S_j^{j_1}\rhd S_j^{j_2}\rhd S_j^{j_3}$. Let the order $\mo$ over agents be the following. 
	$$\mo=L_{q/3+1}\rhd L_{q/3+2}\rhd\cdots\rhd  L_{t}\rhd X\rhd a\rhd C\rhd L_1\rhd\cdots \rhd L_{q/3}$$
	The balanced alternation policy is thus $\mo\rhd \text{inv}(\mo)$, where $\text{inv}(\mo)$ is the inverse order of $\mo$. It is not hard to verify that in the first round the allocation w.r.t.~$\mo$ is as follows:
	\begin{itemize}
	\item for each $j\geq q/3+1$, agent $S_j$ gets item $S_j$ and agent $S_j^{j_s}$ gets item $S_j^{j_s}$;
	\item for each $i\leq q$, agent $x_i$ get $S_j^i$ for the (only) $j\leq q/3$ such that $x_i\in S_j$;
	\item agent $a$ gets item $a$;
	\item for each $k\leq q/3$, agent $c_k$ gets item $S_k$;
	\item for each $j\leq q/3$ and $s=1,2,3$, agent $S_j$ gets an item in $D$ and agent $S_j^{j_s}$ gets an item in $D$.
	\end{itemize}

	In the second round, the allocation w.r.t.~$\text{inv}(\mo)$ is as follows:
	\begin{itemize}
	\item for each $j\leq q/3$ and $s=1,2,3$, agent $S_j$ gets an item in $D$ and agent $S_j^{j_s}$ gets an item in $D$; all items in $D$ ($|D|=8q/3$) are allocated;
	\item for each $k\leq q/3$, agent $c_k$ gets an item in $E$; all items in $E$ are allocated ($|E|=q/3$).
	\item agent $a$ gets item $b$;
	\item other agents get the remaining items.\end{itemize}
	Specifically, agent $a$ gets $\{a,b\}$.

	Now suppose the constructive control has a solution, namely there exists an order $\mo$ over $A$ such that in the sequential allocation w.r.t.~$\mo\rhd \text{inv}(\mo)$ agent $a$ gets $\{a,b\}$. We next show that the {\sc X3C} instance has a solution. For convenience, we divide the sequential allocation of $\mo\rhd \text{inv}(\mo)$ into three stages:
	\begin{itemize}
	\item {\bf Stage 1:} turns before agent $a$'s first turn, where each agent ranked before agent $a$ in $\mo$ chooses an item;
	\item {\bf Stage 2:} turns between agent $a$'s first turn and agent $a$'s second turn, where each agent ranked after agent $a$ in $\mo$ chooses two items;
	\item {\bf Stage 3:} turns after agent $a$'s second turn, where each agent ranked before agent $a$ in $\mo$ chooses an item.
	\end{itemize}
	\begin{claim}\label{claim:C} Agents in $C$ must be after agent $a$ in $\mo$, and they get at least $q/3$ items in $\ms$.
	\end{claim}
	\begin{proof} Because any agent in $C$ ranks item $a$ at their top, all of them must be after agent $a$ in $\mo$. We note that $|C|=q/3$, $|E|=q/3$, and each agent in $C$ will choose two items before agent $a$'s second turn. Therefore, agents in $C$ must get at least $q/3$ items in $\ms$, otherwise one of them will choose $b$, which contradicts the assumption that agent $a$ gets $b$.
	\end{proof}
	W.l.o.g.~let $\{S_1,\ldots,S_{q'}\}$ (for some $q'\geq q/3$) be the items in $\ms$ that are chosen by agents in $C$.
	\begin{claim}\label{claim:s} $q'=q/3$. For all $j\leq q/3$, agents in $\ms_j$ are ranked after agent $a$ in $\mo$, and for all $j\geq q/3+1$, agents in $\ms_j$ are ranked before agent $a$ in $\mo$.
	\end{claim}
	\begin{proof}
	Let $K=\bigcup_{j\leq t}\ms_j\cup D$ denote the set of $4t+8q/3$ items. The crucial observation is that for any agent $s\in \bigcup_{j\leq t}\ms_j$, if $s$ is ranked before $a$ in $\mo$, then in the sequential allocation she will get at least one item in $K$, because she picks an item in $K$ in Stage $1$, and maybe another item in $K$ in Stage $3$; and if $s$ is ranked after $a$ in $\mo$, then in the sequential allocation she will get exactly two items in $K$ in Stage 2. Moreover, each agent in $X$ must get at least one item in $K$ and agents in $C$ must get at least $q/3$ items in $K$. Therefore, agents in $\bigcup_{j\leq t}\ms_j$ get no more than $4t+4q/3$ items in $K$. Because $|\bigcup_{j\leq t}\ms_j|=4t$, at most $4q/3$ of these agents are ranked after $a$ in $\mo$.

	On the other hand, for all $j\leq q'$, agents in $\ms_j$ must be ranked after all agents in $C$ in $\mo$, otherwise some item $S_j$ would have been allocated to an agent in $\ms_j$ (because all of them rank item $S_j$ at the top). By Claim~\ref{claim:C} all agents in $C$ must be ranked after agent $a$ in $\mo$, which means that for all $j\leq q'$, all agents in $\ms_j$ are ranked after agent $a$ in $\mo$. Because $q'\geq q/3$, we must have that $q'=q/3$ and for all $j\leq q/3$, agents in $\ms_j$ are ranked after agent $a$ in $\mo$, and for all $j\geq q/3+1$, agents in $\ms_j$ are ranked before agent $a$ in $\mo$.
	\end{proof}

	Finally, we are ready to show that $\{S_1,\ldots, S_{q/3}\}$ is an exact cover of $X$. For the sake of contradiction suppose $x_i$ is not covered. Let $S_j^i$ (with $j>q/3$) denote an item that agent $x_i$ gets in the sequential allocation. Because agents in $\ms_j$ are before $a$ in $\mo$, it follows that agent $S_j^i$ must get item $S_j$ (because her top-ranked items are $S_j,S_j^i,a$). However, in this case agent $S_j$ must be allocated item $a$, which contradicts the assumption that agent $a$ gets item $a$. Therefore, $\{S_1,\ldots, S_{q/3}\}$ is an exact cover of $X$. This proves the top-$2$ \possibleset is NP-complete.

	We note that item $c$ is the most undesirable item for all agents except agent $a$, which means that agent $a$ gets item $c$ if and only if she does not get item $a$ and $b$. This proves that the \necessaryitem is coNP-complete. 
	\end{proof}

	\section*{Proof of Theorem~\ref{thm:knsa}}

	\begin{proof} Membership in coNP is obvious. We prove that  top-$k$ \necessaryset for $k=2$ is coNP-hard by a reduction from \possibleitem for $k=1$, which is NP-complete~\citep{SaSe13a}. Hardness for other $k$'s can be proved similarly  by constructing preferences so that the distinguished agent always get her top $k-2$ items. Let $(A,I,P,a_1,o)$ denote an instance of possible allocation problem for $k=1$, where $A=\{a_1,\ldots,a_n\}$, $I=\{o_1,\ldots,o_n\}$, $o\in I$, $P=(P_1,\ldots,P_n)$, and we are asked wether it is possible for agent $a_1$ to get item $o$ in some sequential allocation. Given $(A,I,P,a_1,o)$, we construct the following top-$2$ \necessaryset instance.

	{\bf Agents:} $A'=A\cup \{a_{n+1}\}$.

	{\bf Items:} $I'=I\cup \{c_1,c_2\}\cup D$, where $|D|=n$. We have $|I'|=2n+2$.

	{\bf Preferences:} 
	\begin{itemize}
	\item The preferences of $a_1$ is obtained from $P_1$ by inserting $c_2$ right after $o$, and then append $D\succ c_1$.
	\item For any $j\leq n$, the preferences of $a_j$ is obtained from $[P_j\succ D\succ c_2\succ c_1]$ by switching $o$ and $D$.
	\item The preferences for $a_{n+1}$ is $[c_1\succ c_2\succ \text{others}\succ o]$.
	\end{itemize}

	{\bf We are asked} whether agent $a_{n+1}$ always gets items $\{c_1,c_2\}$, which are her top-$2$ items.

	If $(A,I,P,a_1,o)$ has a solution $\mo$, we show that the top-$2$ \necessaryset instance is a ``No'' instance by considering $\mo'=a_{n+1}\rhd\mo\rhd\mo \rhd a_{n+1}$. In the first phase of $\mo'$, $a_{n+1}$ gets $c_1$ and $a_1$ gets $o$. In the third phase $a_{1}$ gets $c_2$.

	Suppose the top-$2$ \necessaryset instance is a ``No'' instance and agent $a_{n+1}$ does not get $\{c_1,c_2\}$ in an balanced alternation policy $\mo'$.  It is easy to see that $a_{n+1}$ must  get $c_1$ in the first phase. Suppose $a_1$ does not get $o$ in the first phase, then in the beginning of the second phase both $o$ and $c_2$ are still available. In this case $a_{n+1}$ must get $c_2$, because clearly none of $a_2$ through $a_n$ can get $c_2$, which means that $a_1$ must get $c_2$ in the second phase. However, this means that $o$ must be chosen by another agent before, which is impossible since it is ranked in the bottom position after $c_1$ and $c_2$ are removed by all other agents. Let $\mo^*$ denote a linear order over $A$ obtained from the restriction of the first phase of $\mo'$ on $A$ by moving all agents who choose an item in $D$ after agent $a_1$ without changing other orders. It is not hard to see that $\mo^*$ is a solution to $(A,I,P,a_1,o)$. \end{proof}

	 \end{document}